\documentclass[sigconf]{acmart}
\usepackage{Definitions2}

\usepackage{microtype}
\usepackage{graphicx}
\usepackage{subfigure}
\usepackage{booktabs} 
\usepackage{multirow}
\usepackage{multicol}
\usepackage{times}
\usepackage{epsfig}
\usepackage{caption}
\usepackage{amsthm}
\usepackage{amsmath}
\usepackage{amssymb}
\usepackage{bm}
\usepackage{enumitem}
\usepackage{makecell}
\usepackage{wrapfig}
\usepackage{indentfirst}
\usepackage{verbatim}
\usepackage{color}
\usepackage{setspace}
\usepackage{attrib}
\usepackage{epigraph}
\usepackage{url}
\usepackage{float}
\usepackage{ragged2e}
\usepackage{algorithm}
\usepackage{algorithmic}
\usepackage[vlined,ruled,resetcount,algo2e]{algorithm2e}
\usepackage[para]{footmisc}
\usepackage{graphicx}

\renewcommand{\textrightarrow}{$\rightarrow$}

\title{Variational Optimization for the Submodular Maximum Coverage Problem}

\author{Jian Du}
\affiliation{
  \institution{Ant Financial}
  \city{Sunnyvale}
  \state{CA}
}
\email{jian.d@antfin.com}

\author{Zhigang Hua}
\affiliation{
  \institution{Ant Financial}
  \city{Sunnyvale}
  \state{CA}
}
\email{z.hua@antfin.com}

\author{Shuang Yang}
\affiliation{
  \institution{Ant Financial}
  \city{Sunnyvale}
  \state{CA}
}
\email{shuang.yang@antfin.com}

\begin{abstract}
We examine the \emph{submodular maximum coverage problem} (SMCP), which is related to a wide range of applications. We provide the first variational approximation for this problem based on the Nemhauser divergence, and show that it can be solved efficiently using variational optimization. The algorithm alternates between two steps: (1) an E step that estimates a variational parameter to maximize a parameterized \emph{modular} lower bound; and (2) an M step that updates the solution by solving the local approximate problem. We provide theoretical analysis on the performance of the proposed approach and its curvature-dependent approximate factor, and empirically evaluate it on a number of public data sets and several application tasks.
\end{abstract}

\begin{document}
\maketitle

\section{Introduction}
Submodular optimization lies at the core of many data mining and machine learning problems, ranging from 
summarizing massive data sets \cite{badanidiyuru2014streaming, karbasi2018data}, cutting and segmenting images \cite{boykov2001interactive,kohli2008p3, jegelka2011submodularity},
 monitoring network status \cite{leskovec2007cost, krause2008near}, diversifying  recommendation systems \cite{mirzasoleiman2016fast,ashkan2015optimal}, searching neural network architectures \cite{xiong2019resource}, interpreting machine learning models \cite{elenberg2017streaming,lakkaraju2016interpretable, kim2016examples,ribeiro2016should}, to asset management and risk allocation in finance \cite{Acerbi12, ohsaka2017portfolio}. Recent works have studied the optimization of submodular functions in various forms, for example,  weighted coverage functions \cite{feige1998threshold}, rank functions of  matroids \cite{krause2009simultaneous}, facility location functions \cite{krause2008efficient}, entropies \cite{sharma2015greedy}, as well as mutual information \cite{krause2012near}. In a typical setting, the optimization is subject to the classical {\it cardinality} constraint, where the number of elements selected is required to be under a preset constant limit. It's been shown that even with this simple constraint, many submodular optimization problems are NP-hard, although under certain conditions
the greedy algorithm can provide a good approximate solution  \cite{nemhauser1978analysis, vondrak2008optimal, wolsey1982analysis}.

The forms of constraints in real applications are often very complex and may be given either analytically or in terms of value oracle models. We, therefore, investigate a more generalized formulation, i.e., the problems of maximizing a submodular function $g(X)$ subject to a general submodular upper bound constraint $f(X)\leq b$.
This problem is referred to as the \emph{submodular maximum coverage problem} (SMCP), or \emph{submodular maximization with submodular knapsack constraint} \cite{iyer2013submodular}. The pioneer work \cite{iyer2013submodular} first examined this problem and introduced an algorithms with bi-criterion approximation guarantees. 
The importance of SMCP has been widely recognized as it can be regarded as a meta-problem for a breadth of tasks including training the most accurate classifier subject to process unfairness 
constraints \cite{grgic2018beyond}, automatically design convolutional neural networks to maximize accuracy with a given forward
time  constraint
\cite{hu2019automatically}, and
 selecting leaders  in a social network for shifting opinions \cite{yi2019shifting}, to name a few. 
 
While \cite{iyer2013submodular} shows the greedy method with a modular approximation has good performance, we take a step further to build a mathematical connection  between the variational modular approximation to a submodular function based on Namhauser divergence and classical variational approximation based on Kullback–Leibler divergence.
We take advantage of this framework 
to iteratively solve SMCP, leading to a novel variational approach. Analogous to the counterpart of variational optimization based on Kullback-Leibler divergence, the proposed method consists of two alternating steps, namely estimation (E step) and maximization (M step) to monotonically improve the performance in an iterative fashion. 
We provide theoretical analysis on the performance of the proposed variational approach and prove that the E step provides the optimal estimator for the subsequent M step. More importantly, we show that the approximate factor of the EM algorithm is decided by the curvature of the objective function and the marginal gain of the constraint function.
We evaluated the proposed framework on a number of public data sets and demonstrated it in several application tasks.

\section{Problem Definition}
\subsection{Formulation}
\emph{Submodularity} is an important property that naturally exists in many real-world scenarios, for example, \emph{diminishing returns} in economics \cite{smith1937wealth}, which refers to the phenomenon that the marginal benefit of any given element tend to decrease as more elements are added. Formally, let $[n]=\{1,2,\ldots, n\}$ be a finite ground set and the set of all subsets of $[n]$ be $2^{[n]}$. The real-valued discrete set function  $f: 2^{[n]}\to \mathbb R$ is  \textit{submodular} on $[n]$  if 
\begin{equation}\label{submodularity1}
f(X) + f(Y)\geq f(X\cup Y) + f(X\cap Y)
\end{equation}
holds for all $X,Y\subseteq[n]$ \cite{fujishige2005submodular}. 
We denote a singleton set with element $j$ as $\{j\}$ and the \textit{marginal} gain as  $f(j|X)\triangleq f(j\cup X)-f(X)$.
The marginal gain is also known as the discrete derivative of $f$ at $X$ with respect to $j$, and  we use $\Delta_{f}$ to denote the { maximum marginal gain}
at $X=\emptyset$:
\begin{equation}
\label{delta_f}
\Delta_{f} = \max_{j\in[n]}
f(j).
\end{equation}
In terms of the marginal gain,
the submodularity defined in (\ref{submodularity1}) is equivalent to \begin{equation}\label{submodularity2}
f(j|X)\geq f(j| Y),
\ \forall X\subseteq Y\subseteq[n], \ j\notin Y.
\end{equation}
Intuitively, the monotonicity means $f$ won't decrease as $X$ is expanded.
A necessary and sufficient monotone condition for $f$ is that
 all its discrete derivatives are nonnegative, i.e., $f(j|X)\geq 0$, for all $j\notin X$ and $X\subseteq[
n]$.

Our primary interest is the \emph{aa} (\texttt{SMCP}):
\begin{equation}\label{P_max}
\begin{aligned}
\max_{X} 
~ g(X), \quad
\mathrm{s.t.}~~ f(X)\leq b,
\end{aligned}
\end{equation}
where $f(X)$ and $g(X)$ are monotone, and are assumed to be normalized such that $f(\emptyset)=0$ and $g(\emptyset)=0$. Our formulation is general enough with minimal assumptions, the techniques developed in this paper, including the analysis, are applicable to general forms of  $g(X)$ and $f(X)$, including those with analytical forms or given in terms of a  value oracle\footnote{For a given set $X$, one  can query an oracle to find its value $f(X)$ and $g(X)$, and both $f(X)$ and $g(X)$ could be computed by a black box.}. 

Fig~\ref{fig:bipartite_graph} illustrates  a  concrete example of SMCP,
where we are given a bipartite graph  consisting of two kinds of nodes, i.e., square nodes and circle nodes; each circle node is associated with a non-negative value, and each square node represents a singleton; the goal is to select a subset, $X$, out of the square nodes, such that the circle nodes being covered have as much total value (denoted by $g(X)$) as possible yet the number of circle node selected is within a set limit, i.e., $f(X)<b$. 

\begin{figure}[t]
\label{fig:bipartite_graph}
\centering
\includegraphics[width=0.36\textwidth]{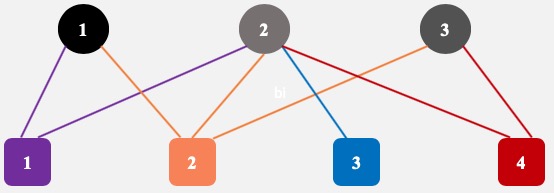}
\caption{A bipartite graph to illustrate an SMCP problem. Each square node  covers at least one circle node.
Different square nodes is allowed to cover the same circle node, and each circle node associates a nonnegative value. The goal is to find a subset of the square nodes covering maximum total value, yet with the  number of circle nodes smaller than $b$.}
\label{fig:bipartite_graph}
\end{figure}

\subsection{Related Problems}
The SMCP problem was first studied in \cite{iyer2013submodular}, where it was also referred to as \emph{submodular cost with submodular knapsack constraint} (SCKC). The authors further established 
the equivalence between SMCP (\ref{P_max}) and minimizing $f(X)$ subject to $g(X)\geq c$ (called \emph{submodular cost with submodular cover constraint} or SCSC). A greedy algorithm and an 
ellipsoidal approximation method were employed to solve SMCP in \cite{iyer2013submodular}.

SMCP is regarded as a meta-problem to many application tasks, of which we introduce a few examples. In \cite{grgic2018beyond}, it was shown that training a classifier with fairness 
constraints involves solving a variant of SMCP, where $X$ is the feature subset, and both the objective (i.e., loss function) and constraints are submodular. \cite{hu2019automatically} studied automatically designing convolutional neural networks (CNNs) to maximize accuracy within a given forward
time  constraint, where $X$ represents the configuration of a CNN (e.g, kernel size at each layer) and $f(X)$ is the forward time function.
It's shown that the validation
accuracy on a held out set of samples is submdoular \cite{xiong2019resource}.
\cite{kempe2003maximizing} investigated influence maximization in social networks
and show the
influence function is submodular, although the formulation is unconstrained. \cite{yi2019shifting} studied French-Degroot opinion dynamics in a
social network with two polarizing parties to shift opinions in a social network through leader selection. In their formulation, $g(X)$ is the influence function and $f(X)$ is the average opinion of all nodes, both of which are submodular. In finance, the  \emph{conditional value at risk} (CVaR) is well known and widely used for risk control and portfolio management, for example, \cite{ohsaka2017portfolio} examined the maximization of CVaR to select portfolios based on a formulation similar to SMCP, and employed a greedy method to solve it.

Cardinality-constrained submodular maximization is a special case of SMCP since the candinality $|X|$ is a modular function. A number of important tasks can be approached by this simpler variant of SMCP, for example, data set summarization \cite{badanidiyuru2014streaming, karbasi2018data}, network status monitoring  \cite{leskovec2007cost, krause2008near}, and interpretable machine learning \cite{elenberg2017streaming, kim2016examples,ribeiro2016should}.

\section{Variational Bounds}
A submodular function resembles both convex functions and concave functions
 \cite{iyer2013submodular}, in the sense that it can be bounded both from above and below.
In this section, we propose \emph{variational SMCP} (V-SMCP), a variational approximate for SMCP based on Nemhauser divergence.

\subsection{\bf Upper Bound for $f(X)$}

In the seminal work \cite{nemhauser1978analysis}, it is demonstrated  that the submodularity of $f(X)$ in (\ref{submodularity1}) is equivalent to the following inequality  
\begin{equation}\label{Nemhauster_divergence1}
\underbrace{f(X)-\sum_{j \in X \backslash Y} f(j | X\backslash j)+\sum_{j \in Y \backslash X} f(j | \Theta)}_{\triangleq \widehat f_{X}(Y;\Theta)}
 \geq f(Y), \ \forall X, Y\subseteq [n],
\end{equation}
where $\Theta=  X \cap Y$.
Following from the above inequality, 
the Nemhauser divergence \cite{iyer2012submodular}  between two set functions $\widehat f_{X}(Y; \Theta)$  and  $f(Y)$ is defined as   
\begin{equation}
\mathcal D(\widehat f_{X}(Y; \Theta)|| f(Y)) = \widehat f_{X}(Y; \Theta) -  f(Y),
\end{equation}
which
 satisfies  $\mathcal D(\widehat f_{X}(Y; \Theta)|| f(Y)) \geq 0$.  The equality holds when $X=Y$, which implies $\Theta= Y$.
The Nemhauser divergence measures the distance between two set functions and is not  symmetric, which is similar to the Kullback-Leibler divergence that measures the distance between two probability distributions.

 Note that \cite{nemhauser1978analysis}  provids another inequality, which is also equivalent to the submodularity of $f(X)$, given by
\begin{equation}
\label{Nemhauster_divergence1}
\underbrace{f(X)+\sum_{j \in Y \backslash X} f(j | X)-\sum_{j \in X \backslash Y} f(j | \Psi\backslash j)}_{\triangleq  \widehat f_{X}(Y;\Psi)}
 \geq f(Y), \ \forall X, Y\subseteq [n],
\end{equation}
with $\Psi = X\cup Y$.
We  can therefore define the divergence with  
$\widehat f_{X}(Y; \Psi)$,
i.e., $\mathcal D(\widehat f_{X}(Y; \Psi)|| f(Y)) = \widehat f_{X}(Y; \Psi) -  f(Y)$ for the variational optimization.
Yet, as there is no guarantee that which one between these two functions provides a better  approximation, we focus on $\mathcal D(\widehat f_{X}(Y; \Theta)|| f(Y))$ in this paper, and all the algorithms and analyses provided can be adapted to the algorithm based on $\mathcal D(\widehat f_{X}(Y; \Psi)|| f(Y))$.

\subsection{Lower Bound for $g(X)$}

We define a permutation on the elements of $[n]$, i.e., $\pi:[n]\to [n]$ that orders the elements in $[n]$ as a sequence $\left(\pi_1, \pi_2, \ldots, \pi_n\right)$,
which denotes that if $\pi_i=j$,  $j$ is the $i$-th
element in this sequence.
Particularly, given a subset $X_t\subseteq 2^{[n]}$, we choose a permutation $\pi$ that places the elements in $X_t$ first and then includes the remaining elements in $[n]\backslash X_t$, where the subscript $t$ denotes the iteration number used in the EM algorithm introduced in the next section.
We further define the corresponding sequence of subsets of $[n]$ as $S^{\pi}_i$ with $i=0,\ldots, n$, which is given by
\begin{equation} \label{setchain}
S^{\pi}_0=\emptyset, \
S^{\pi}_1=\{\pi_1\}, \ \ldots, \
S^{\pi}_n=\{\pi_1,\ldots, \pi_n\},
\end{equation}
which results in
$
\emptyset=S^{\pi}_{0} \subset S^{\pi}_{1} \subset S^{\pi}_{2} \ldots \subset S^{\pi}_{n}=[n] 
$. 
Then a lower bound of $g(X)$ is given by \cite{iyer2013submodular}
\begin{equation}\label{loewr_bound}
\widehat g^{\pi}_{X_t}(X)=  \sum_{j\in X}\widehat g^{\pi}_{X_t}(j), \ \forall X\subset [n],
\end{equation}
where $\widehat g^{\pi}_{X_t}(j)$ with $j=\pi_i$ is defined by \cite{iyer2013submodular}
\begin{equation}
\label{diff-sub}
\widehat g^{\pi}_{X_t}(j) =
\widehat g^{\pi}_{X_t}(S_i^{\pi}-S^{\pi}_{i-1}) = g(S_i^{\pi})- g(S^{\pi}_{i-1}).
\end{equation}
Since $X_t$ and $\pi$ has a mapping relationship, in the following of the paper, we omit the superscript $\pi$ when no confusing is caused.
The lower bound property, i.e., 
$\widehat g_{X_t}(X)\leq g(X)$ can be easily proved \cite{iyer2013submodular} according to the submodularity.
Further more, substituting (\ref{diff-sub}) into (\ref{loewr_bound}) and considering the permutation given by $\pi$, it guarantees the tightness at $X_t$ that
\begin{equation}\label{tight_g}
    \widehat{g}_{X_t}(X_t) = g(X_t).
\end{equation}
 
\subsection{Variational Approximation for SMCP}
The SMCP in (\ref{P_max}) can be approximated, at any given $X_t$, by the following problem, which we call \texttt{V-SMCP}:
\begin{equation}\label{P_max_modular}
\begin{aligned}
 \max_{X}~~~~ 
& \widehat g_{X_t}(X ) \\
\quad\mathrm{s.t.}~~~~
& \widehat f_{X_t}(X;\widehat\Theta_t)\leq b,\\
& \widehat\Theta_t = \argmin_{\Theta} \mathcal D(\widehat f_{X_t}(X;\Theta)
||f(X)),\\
& \Theta = X\cap X_t,
\end{aligned}
\end{equation}
where $\widehat g_{X_t}(X )$ and $\widehat f_{X_t}(X;\widehat\Theta_t)$ are lower bound and upper bound for $g(X)$ and $f(X)$, respectively.
V-SMCP is an effective approximation of SMCP as both bounds are tight at $X_t$, i.e.,
$\widehat g_{X_t}(X_t ) = g(X_t)$ and 
$\widehat f_{X_t}(X_t;X_t)=f(X_t)$.


\section{Variational Optimization}
\label{VariationlEM}

In this section, we introduce an iterative method to solve an SMCP based on a sequence of V-SMCPs. It alternates between (1) an \emph{estimation} (E) step that minimizes the Namhauser divergence by estimating the parametric approximation; and (2) a subsequent \emph{maximization} (M) step that updates the solution.

Since $\widehat f_{X_t}(X;\Theta)$ is an upper bound of $ f(X)$, maximizing $ \widehat f_{X_t}(X; \Theta)$ w.r.t $\Theta$ will equivalently minimizing $\mathcal D_{\Theta}(\widehat f_{X_t}(X;\Theta)|| f(X))$.
We therefore 
treat  $\Theta$ as a variational parameter and estimate it in the E step to reduce $\mathcal D_{\Theta}(\widehat f_{X_t}(X;\Theta)|| f(X))$ as much as possible. Then with $\Theta=\widehat\Theta_t$, we update the solution by solving a V-SMCP in the M step. We name this method \emph{estimation-maximization} (EM) algorithm.

\subsection{\bf{E step: Estimate $\widehat \Theta_t$}}

According to  the submodularity definition in (\ref{submodularity2}), we have $f(j|\Theta_1)\geq f(j|\Theta_2)$ if $\Theta_1\subseteq \Theta_2$.  Following from (\ref{Nemhauster_divergence1}),
for all $ X\subseteq [n]$, we further obtain
\begin{equation}\label{monotone}
\widehat f_{X_t}(X;\Theta_1)\geq  \widehat f_{X_t}(X;\Theta_2), \
\forall  \Theta_1\subseteq \Theta_2.
\end{equation}
This inequality indicates that 
 we can decrease the divergence of $\mathcal D(\widehat f_{X_t}(X;\Theta)|| f(X))$ by enlarging  $\Theta$.
 Thus, the largest $\Theta$ is $X_t$, according to the Nemhauser divergence defined in (\ref{Nemhauster_divergence1}).
To avoid notational clumsiness,  we use $E\backslash j$ to denote a set that excludes $j$, i.e.,
\begin{equation} \label{abb_E}
E{\backslash j}=\left\{
\begin{aligned}
&X_t \backslash j,   & \textrm{if} \ j\in X_t, \\
&X_t, & \textrm{if}\ j\notin X_t.
\end{aligned}
\right.
\end{equation}
By substituting (\ref{abb_E}) to (\ref{Nemhauster_divergence1}), we  define a 
permutation operation $\epsilon:[n]\to [n]$ that orders the elements in $[n]$ as a sequence $\left(\epsilon_1, \epsilon_2, \ldots, \epsilon_n\right)$ such that \begin{equation}\label{sorting_1}
\frac{\widehat g_{X_t}(\epsilon_1)}{f(\epsilon_1|  E\backslash{\epsilon_1})}
\geq 
\frac{\widehat g_{X_t}(\epsilon_2)}{f(\epsilon_2|  E\backslash{\epsilon_2})}
\geq
\ldots
\geq
\frac{\widehat g_{X_t}(\epsilon_n)}{ f(\epsilon_n|  E\backslash{\epsilon_n})}.
\end{equation}
There must exist a $\widehat k$ such that $$\widehat k=\argmax_{k}\sum_{k^{\prime}=1}^{k} 
f(\epsilon_{k^{\prime}}| E_{\epsilon_k})\leq b.$$ 
We then obtain an estimation of $\Theta$ given by 
\begin{equation} \label{estimation}
\widehat \Theta_t =  X_t\cap \widehat X_t,
\end{equation}
with
\begin{equation}\label{inter-solution}
\widehat X_t = \{\epsilon_1,\epsilon_2,\ldots, \epsilon_{\widehat k} \}.
\end{equation}
While there is no guarantee that $\widehat f_{X_t}(\widehat X_t;\widehat\Theta_t)\leq b$ is satisfied, the estimator
$\widehat \Theta_t$  as well as $\widehat X_t$  would lead to a larger feasible space for maximizing $\widehat g_{X_t}(X)$
in the subsequent M step, which is analytically proved in Section \ref{analysis}. 
 
\subsection{\bf { M Step: Compute the Maximizer $X_{t+1}$}}

For notational brief, in the M step,
we  represent a set without $j$ given $X_t$ as
\begin{equation}
\label{abb_M}
M\backslash j =\left\{
\begin{aligned}
&X_t \backslash j,   \ \textrm{if} \quad  j\in X_t, \\
& \widehat \Theta_t,  \quad \textrm{if}\quad j\notin X_t.
\end{aligned}
\right.
\end{equation}
Substituting (\ref{abb_M}) to (\ref{Nemhauster_divergence1}), we further   define a new
permutation $\mu:[n]\to [n]$ that orders the elements in $[n]$ as a new sequence $\left(\mu_1, \mu_2, \ldots, \mu_n\right)$ such that 
\begin{equation}
\label{sorting2}
\frac{\widehat g_{X_t}(\mu_1)}{ f(\mu_1|M\backslash{\mu_1})}
\geq 
\frac{\widehat g_{X_t}(\mu_2)}{ f(\mu_2|M\backslash{\mu_2})}
\ldots
\geq
\frac{g_{X_t}(\mu_n)}{ f(\mu_n|M\backslash{\mu_n})}.
\end{equation}
By letting $\widehat m$ be the largest index that satisfy the following inequality:
\begin{equation}
\label{m_max}
\widehat m=\argmax_{m}\sum_{m^{\prime}=1}^{m} 
f(\mu_{m^{\prime}}|M_{\mu_{m^{\prime}}})\leq b,
\end{equation}
we finally obtain the optimizer at the $t$-th iteration:
\begin{equation}\label{solution}
X_{t+1} = \{\mu_1,\mu_2,\ldots, \mu_{\widehat m} \}.
\end{equation}
From (\ref{loewr_bound}), the corresponding  objective value is
\begin{equation}
\label{obj_t}
    \widehat g(X_{t+1})=
    \sum_{m^{\prime}=1}^{\widehat m}g_{X_t}(\mu_{m^{\prime}}).
\end{equation}

The algorithm terminates once 
$
\widehat g(X_{t+1})\leq \widehat g(X_{t})$, which is equivalent to  
$g(X_{t+1})\leq g(X_{t})$ according to (\ref{tight_g}).
The proposed EM algorithm is summarized in Algorithm 1. Note that in both E step and M step, the permutation $\epsilon$ and $\mu$ can be implemented in   $\mathcal{O}(n\log{}n)$ time
through any efficient sorting procedure.

Fig.~ \ref{fig:visual_EM}
shows how the EM algorithm approximates the solution of P1  in the space of $2^{[n]}\times \mathbb R $. The black curve represents the objective function under constraint in SMCP. 
At the $t$-th iteration, we construct $\widehat g_{X_t}(X)$ with tightness guarantee at $X_t$ according to (\ref{tight_g}).
In the E step, we compute $\widehat \Theta_t$ to enlarge the feasible space, in the subsequent M step, we compute $X_{t+1}$, which is the approximate solution for P2. 
The corresponding function is shown by the red curve.
Then at the $(t+1)$-th, we compute the new lower bound with estimation $\widehat\Theta_{t+1}$ depicted by the blue color. 

A simplified version of EM algorithm can be obtained by setting $\widehat \Theta_t = \emptyset$ in the EM algorithm. This simplified EM (SEM) method saves the computation cost for
the permutation $\epsilon$ in the E step. We summarize the SEM in Algorithm 2.
However, it is evident that the E step of EM algorithm leads to larger or equal (when $\widehat \Theta = \emptyset$ in (\ref{estimation})) feasible space than the SEM algorithm. Therefore, it is guaranteed that the EM algorithm has a no smaller objective value than SEM, which is also verified by experiments in Section \ref{section:experiments}.

\begin{algorithm}[t]
\label{alg1}
\caption{EM algorithm.}
\begin{algorithmic}[1]
\REQUIRE Initialization: $X_0$
\COMMENT{EM Algorithm.}
\WHILE{$\widehat g_{X_t}(X_t)\leq g_{X_t}(X_{t+1})$}
\STATE{{\bf E Step:}}
  \STATE{Compute  $\widehat X_t $  via (\ref{inter-solution}).}
  \STATE{Update $\widehat \Theta_t$ via (\ref{estimation}).}
\STATE{{\bf M Step:}}
  \STATE{Update $X_{t+1}$  via (\ref{solution})}.
\ENDWHILE
\end{algorithmic}
\end{algorithm}

\begin{algorithm}[t]
\label{alg-SEM}
\caption{SEM algorithm}
\begin{algorithmic}[1]
\REQUIRE Initialization: $X_0$
\COMMENT{SEM Algorithm.}
\WHILE{$\widehat g_{X_t}(X_t)\leq g_{X_t}(X_{t+1})$}
\STATE{{\bf E Step:}}
  \STATE{Set $\widehat \Theta_t=\emptyset$.}
\STATE{{\bf M Step:}}
  \STATE{Update $X_{t+1}$  via (\ref{solution})}.
\ENDWHILE
\end{algorithmic}
\end{algorithm}

\section{Theoretical Analysis}
\label{analysis}
In this section, we provide analysis of the proposed EM algorithm.
By replacing $X$ and $Y$ in (\ref{Nemhauster_divergence1}) with $X_t$ and $X_{t+1}$, we obtain 
\begin{equation}\label{test}
b\geq 
\underbrace{f(X_t)-\sum_{j \in X_t \backslash X_{t+1}} f(j | X_t\backslash j)+\sum_{j \in X_{t+1} \backslash X_t} f(j | \Theta)}_{= \widehat f_{X_t}(X_{t+1};\Theta)}
 \geq f(X_{t+1}).
\end{equation}

\begin{figure}[t]
\centering
\includegraphics[width=0.45\textwidth]{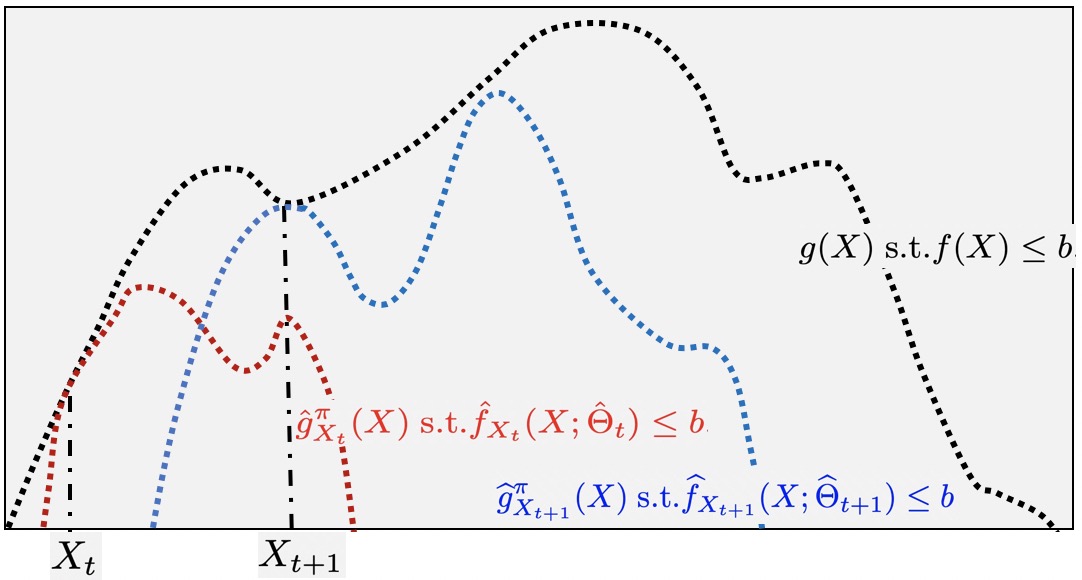}
\caption{The alternating process of the proposed EM algorithm, which involves computing a lower bound with maximum feasible space and then maximizing this lower bound to update the optimizer.}
\label{fig:visual_EM}
\end{figure}
The quantities $\Theta = X_t\cap X$ implies that $\emptyset\subseteq\Theta\subseteq X_t$.
According to  (\ref{monotone}), we have 
$\widehat f_{X_t}(X;\emptyset) \geq \widehat f_{X_t}(X;\Theta)\geq \widehat f_{X_t}(X;X_t)$.
Thus, $\widehat f_{X_t}(X;X_t)$ is the tightest bound we can achieve.
In spite of this, we cannot simply set $\widehat \Theta_t = X_t$ since it is unsecured  that   $X_{t+1}$, which is obtained in the M step, satisfies $X_{t}\cap X_{t+1}= X_t$.
Then there is no warranty that $\widehat f_{X_t}(X_{t+1};X_t)\geq f(X_{t+1})$. Consequently, $f(X_{t+1})<b$ is not guaranteed, and $X_{t+1}$ may not lies in the feasible space, which violates the constraint in (\ref{P_max_modular}).
In the following theorem, we analytically show the {\it optimality} of  $\widehat \Theta_t$ (equation (\ref{estimation}) in the proposed  E step).
Here, an optimal $\widehat \Theta_t$ implies that it provides the feasible space which is a superset of all the feasible space provided by any other $\Theta_t$'s.

\begin{theorem}
[{\bf Optimality}]
\label{EM_opt}
In the E step of the EM algorithm, $\widehat \Theta_t$, i.e., equation (\ref{estimation}), provides the optimal  $\Theta_t$ for the optimization problem in the M step at each iteration.
\end{theorem}
\begin{proof}
Because in the E step, we have no idea about the $X_{t+1}$, we need to estimate a $\widehat \Theta_t$ such that for all possible $X_{t+1}$, the constraint $f(X_{t+1})\leq \widehat f_{X_t}(X;\Theta)$ is always satisfied, so that 
$f(X_{t+1})\leq b$ is guaranteed.
Thus, according to $(\ref{Nemhauster_divergence1})$, we need to show that $f(j|\widehat \Theta_t)$ is larger than any  $f(j|\Theta_t)$. We prove this as follows.

First, according to   (\ref{estimation}), we have
 $\emptyset\subseteq \widehat\Theta_t \subseteq X_t$.
Then, (\ref{monotone}) shows that setting $\widehat\Theta_t = X_t$ leads to the smallest $f(\epsilon_i|E\backslash\epsilon_i)$.
At last, due to the sorting mechanism in (\ref{sorting_1}), the solution $\widehat X_t$ is the smallest subset containing elements from all possible $X_{t+1}$, which makes any $\Theta_t$ that satisfies
$\emptyset \subseteq \Theta_t \subseteq \widehat X_t\cap X_t$
is a subset of $ X_{t+1}\cap X_t$. 
Thus, we conclude that  $\forall X_{t+1}$,  $f(X_{t+1})\leq \widehat f_{X_t}(X_{t+1};\Theta_t)$ if $\emptyset \leq \Theta_t \leq \widehat X_t\cap X_{t}$.
Hence, in the feasible range, the optimal $\Theta$ is obtained by setting  $\widehat \Theta_t = \widehat X_t\cap X_{t}$ as it gives the smallest $\widehat f_{X_t}(X_{t+1};\Theta_t)$, i.e.,
$$\widehat f_{X_t}(X_{t+1};
\widehat\Theta_t)\leq 
\widehat f_{X_t}(X_{t+1};\Theta_t)$$
for all the feasible $\Theta_t$.
Therefore, it  leads to the largest feasible space.
\end{proof}
 
\begin{proposition}
[{\bf Monotonicity}]
The EM algorithm  monotonically improves the objective function value, i.e., $g(X)$ in the feasible space of SMCP, i.e., $g(X_0)\leq g(X_1)\leq
g(X_2)\leq \ldots$. 
\end{proposition}
\begin{proof}
According to the tightness property in (\ref{tight_g}), we have
$g(X_t)=\widehat g_{X_t}(X_t)$.
Since the proposed EM algorithm leads to  increment of $\widehat g_{X_t}(X)$ at each iteration, we obtain $\widehat g_{X_t}(X_t)\leq \widehat g_{X_t}(X_{t+1})$. 
Moreover, the lower bound property of $\widehat g_{X_t}(X)$ results in 
$\widehat g_{X_t}(X_{t+1})\leq g(X_{t+1})$. Therefore, we obtain
$g(X_t) \leq  g(X_{t+1})$.
Hence, the monotonicity of the EM algorithm is proved.
\end{proof}

We next provide tightened, curvature-dependent approximation ratio for the proposed algorithms.
Curvature has
served to improve the approximation ratio for submodular maximization problems, e.g., from 
$(1-\frac{1}{e})$ to 
$\frac{1}{\kappa_g}(1-e^{-\kappa_g} )$
for monotone submodular maximization subject to a cardinality constraint \cite{conforti1984submodular} and matroid constraints \cite{vondrak2010submodularity}.
We first give the definition of curvature and then analytically prove our results.

Given a submodular function $g$,   the {\it  curvature} $\kappa_g$, which represents the deviation from modularity, is defined as 
 \begin{equation}
\label{curvature}
\kappa_g=1-\min_{k\in[n]}\frac{g(k|[n]\setminus k)}{g(k)}.
 \end{equation}
 The curvature, $\kappa_g$ measures the distance of $g$ from modularity, and 
$\kappa_g = 0$ if and only if $g$ is modular, i.e., 
$g(X) = \sum_{j\in X} g(j)$.
Next, we show the approximation ratio of $\widehat g(X)$ to $g(X)$ in terms of curvature.
\begin{theorem}
[{\bf Function $\widehat g$ Approximation Ratio\footnote{Theorem  \ref{approximation_ratio_g}
is independent of the constraint $f(X)$, and therefore  it applies to any submodular function.}}]
\label{approximation_ratio_g}
Given arbitrary $\pi$ and $X_t$, the approximation ratio given by $\widehat g_{X_t}(X)$ to $g(X)$ is $1-\kappa_g$, i.e.,
\begin{equation}
\label{g-approximation}
\widehat g_{X_t}(X) \geq
(1-\kappa_g) g(X), \quad \forall X\subseteq [n].
\end{equation}
\end{theorem}
\begin{proof}
The definition of $\widehat g(j)$ in (\ref{diff-sub}) is equivalently represented by $\widehat g_{X_t}(j)=  g(S_j^{\pi})- g(S^{\pi}_{j-1})
= g(j|S^{\pi}_{j-1})
$ for all $j\in [n]$.
Then, according to (\ref{submodularity2}), we have $\widehat g_{X_t}(j) =  g(j|S^{\pi}_{j-1}) \geq g(j|[n]\backslash j)$.
Dividing both sides by a positive number 
$g(j)$, we have
\begin{equation}
\frac{\widehat g_{X_t}(j)}{g(j)}
\geq  
   \frac{g(j|[n]\backslash j)}{g(j)}
\geq 
\min_{k\in[n]}
   \frac{g(k|[n]\backslash k)}{g(k)}, \quad \forall j\in[n].
\end{equation}
The most right-hand side of the above inequality is $1-\kappa_g$
according to the curvature definition in (\ref{curvature}). We therefore obtain
\begin{equation}\label{curverture_ineq}
\frac{\widehat g_{X_t}(j)}{g(j)}
\geq  
 1- \kappa_g, \quad  \forall j\in[n].
\end{equation}

Next, we extend the above inequality from an arbitrary element $j\in [n]$ to an arbitrary set $X\subseteq [n]$ by induction. 
Equation (\ref{curverture_ineq}) implies that
\begin{equation}
\frac{\widehat g_{X_t}(j_1)
+
\widehat g_{X_t}(j_2)}{g(j_1)+
 g(j_2)}
\geq  
 1- \kappa_g, \quad  \forall j_1, j_2\in[n].
\end{equation}
Then by induction, we have
\begin{equation}
\label{modular_ratio}
\frac{\sum_{j\in X}\widehat g_{X_t}(j)
}{\sum_{j\in X}g(j)}
\geq  
 1- \kappa_g, \quad  \forall X\subseteq[n].
\end{equation}
The numerator in the left-hand side of the above inequality is equivalent to $\widehat g_{X_t}(X)$  from (\ref{loewr_bound}).
Furthermore,   the submodularity of $g(X)$ indicates $g(X)\leq \sum_{j\in X}g(j)$.  Therefore, we can  further magnify the left-hand side of (\ref{modular_ratio}) and obtain
\begin{equation}
\frac{\widehat g_{X_t}(X)}
{g(X)}\geq
\frac{\sum_{j\in X}\widehat g_{X_t}(j)
}{\sum_{j\in X}g(j)}
\geq  
 1- \kappa_g,  \quad \forall X\subseteq[n].
\end{equation}
\end{proof}

Next, we show the approximation ratio of the proposed EM/SEM algorithm for $\widehat g_{X_t}(X)$ in a V-SMCP.
Let   $\textrm{OPT}_{\widehat g_{X_t}}$ denote the optimizer for  (\ref{P_max_modular}).

\begin{proposition}
\label{KP}
At each iteration, both the EM and SEM algorithms obtain a set $X_{t+1}$ such that
\begin{equation}
\label{g_hat_approx}
\widehat g_{X_t}(X_{t+1})\geq (1-\frac{2\Delta_{f}}{b}) \widehat g_{X_t}(\textrm{OPT}_{\widehat g_{X_t}}).
\end{equation}
\end{proposition}
The tedious but straightforward proof for this proposition is provided in the Appendix. 
Yet, this proposition paves the way to the proof of the approximation ratio of the EM algorithm for $g(X)$ in V-SMCP.
Let \textrm{OPT}  denote the optimizer for (\ref{P_max}), and with the knowledge of Theorem \ref{approximation_ratio_g} and Proposition \ref{KP} in mind, w be have the following  result.
\begin{theorem}
[{\bf Approximate Optimality}]
\label{approx_ratio}
 The results of both EM and SEM algorithm, i.e., $g(X_{t+1})$ hold the approximation ratio  $(1-\kappa_g)(1-\frac{2\Delta_{f}}{b})$,
 i.e.,
\begin{equation}
g(X_{t+1})\geq (1-\kappa_g)(1-\frac{2\Delta_{f}}{b}) g(\textrm{OPT}).
\end{equation}
\end{theorem}
\begin{proof}
Since $\textrm{OPT}_{\widehat g_{X_t}}$ is the optimizer of $\widehat g_{X_t}(X)$, we have the inequality
 $\widehat g(\textrm{OPT}_{\widehat g_{X_t}})\geq \widehat g(\textrm{OPT})\geq 0$. Further, due to $g(X)\geq 0$ for all $X\subseteq [n]$, and 
following from (\ref{g-approximation}), we  obtain $\widehat g(\textrm{OPT})\geq 0$. 
We then have
$\widehat g_{X_t}(\textrm{OPT}_{\widehat g_{X_t}})\geq \widehat g_{X_t}(\textrm{OPT})\geq 0,$
 which results in 
\begin{equation}
    \frac{g(\textrm{OPT})}{\widehat g_{X_t}(\textrm{OPT}_{\widehat g_{X_t}})}\leq \frac{g(\textrm{OPT})}{\widehat g_{X_t}(\textrm{OPT})}\leq
    \frac{1}{1-\kappa_g},
\end{equation}
where the second inequality is due to Theorem \ref{approximation_ratio_g} by replacing $X$ with $\textrm{OPT}$ in (\ref{g-approximation}).
Then, we obtain $$\widehat g_{X_t}(\textrm{OPT}_{\widehat g_{X_t}})
\geq
(1-\kappa_g)g(\textrm{OPT}).
$$
Substituting (\ref{g_hat_approx}) to the left-hand side of the above inequality, we obtain
$$
\widehat g_{X_t}(X_{t+1})
\geq (1-\kappa_g)(1-\frac{2\Delta_{f}}{b})g(\textrm{OPT}).
$$
\end{proof}

\section{Experiments}
\label{section:experiments}
\begin{figure*}[hbt!]
\centering
\subfigure[PCodes Dataset]{
\begin{minipage}[t]{0.26\linewidth}
\centering
\includegraphics[width=1.8in]{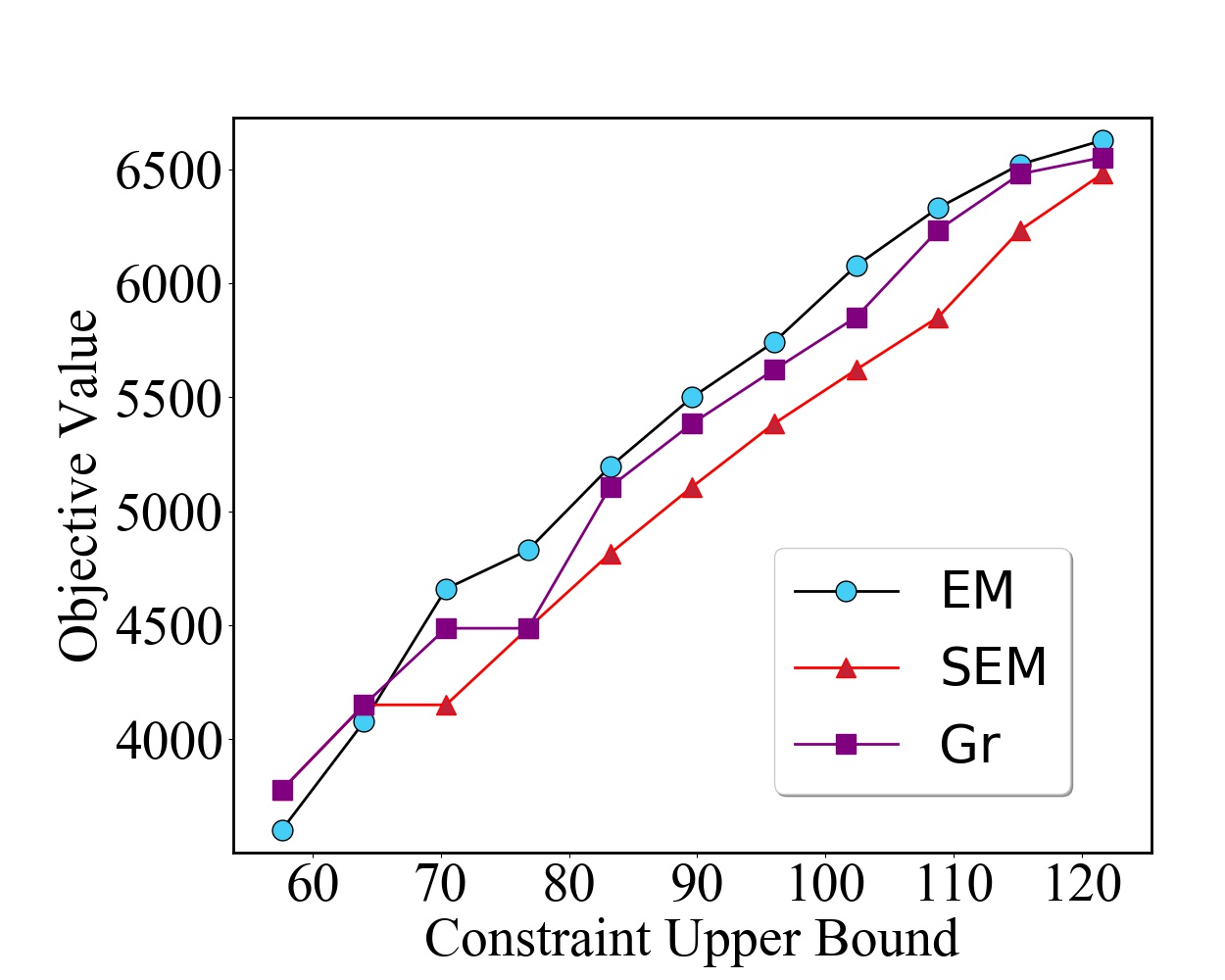}
\end{minipage}%
}%
\subfigure[Gap-A]{
\begin{minipage}[t]{0.26\linewidth}
\centering
\includegraphics[width=1.8in]{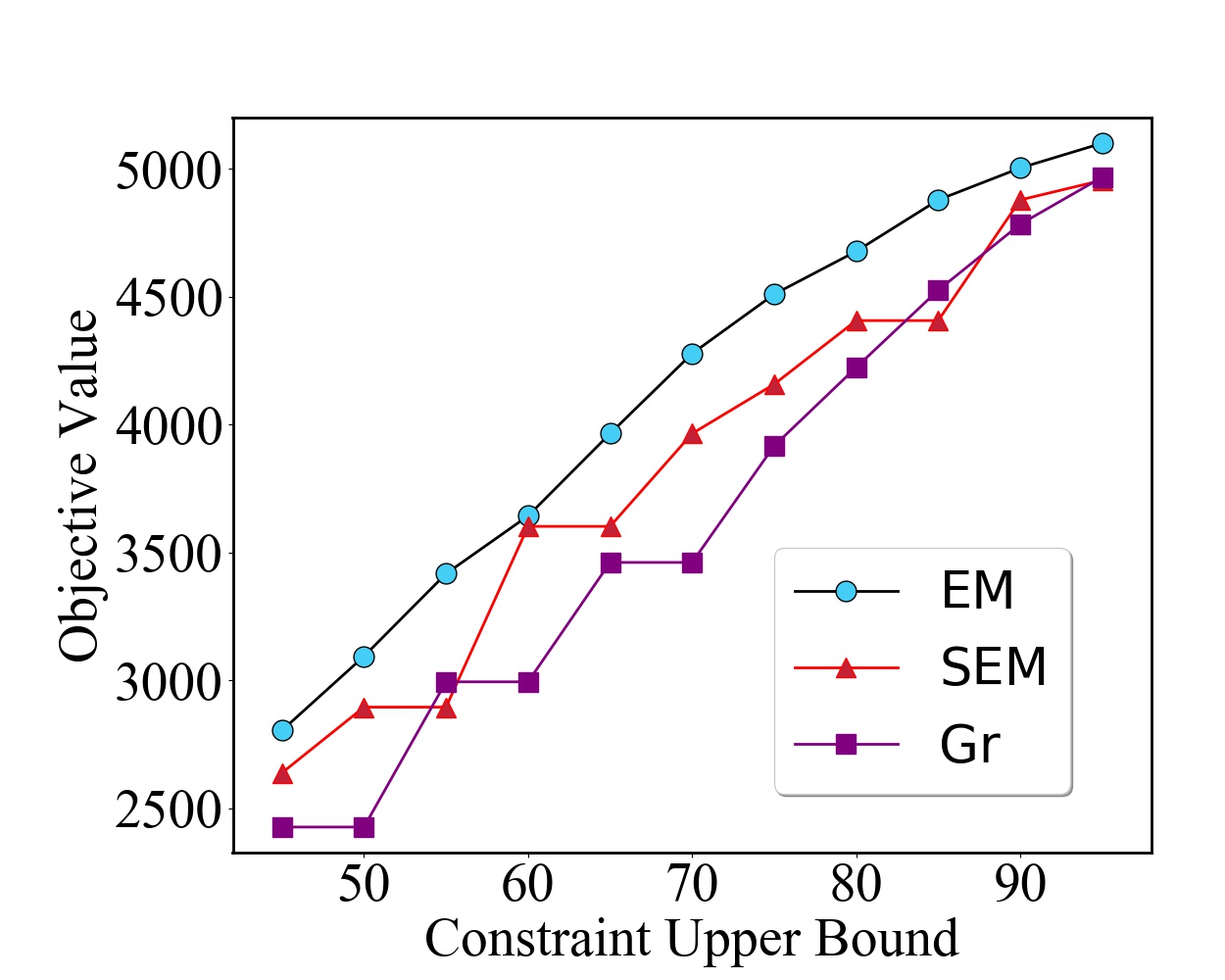}
\end{minipage}%
}%
\subfigure[Chess]{
\begin{minipage}[t]{0.26\linewidth}
\centering
\includegraphics[width=1.8in]{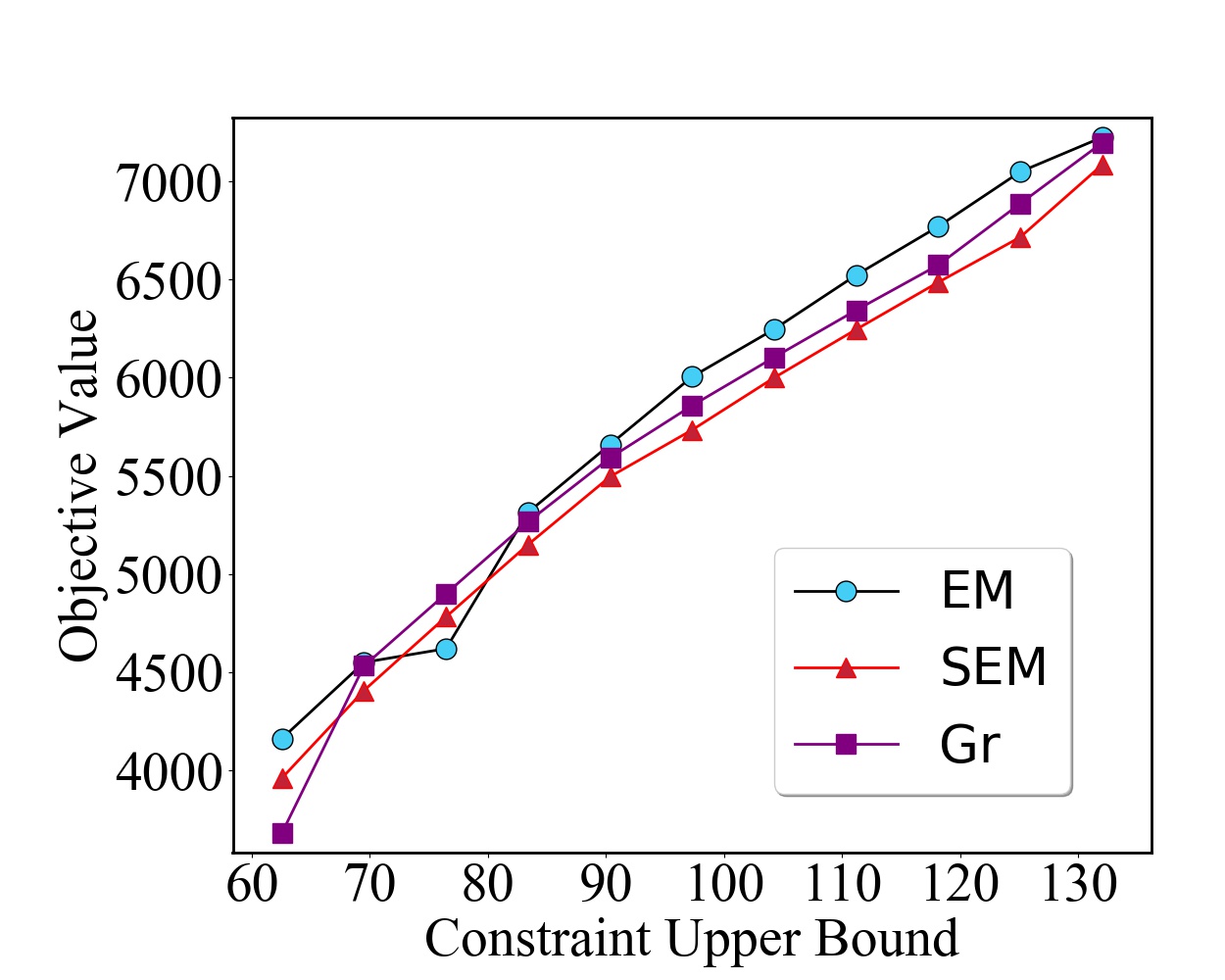}
\end{minipage}
}%
\subfigure[FPP]{
\begin{minipage}[t]{0.26\linewidth}
\centering
\includegraphics[width=1.8in]{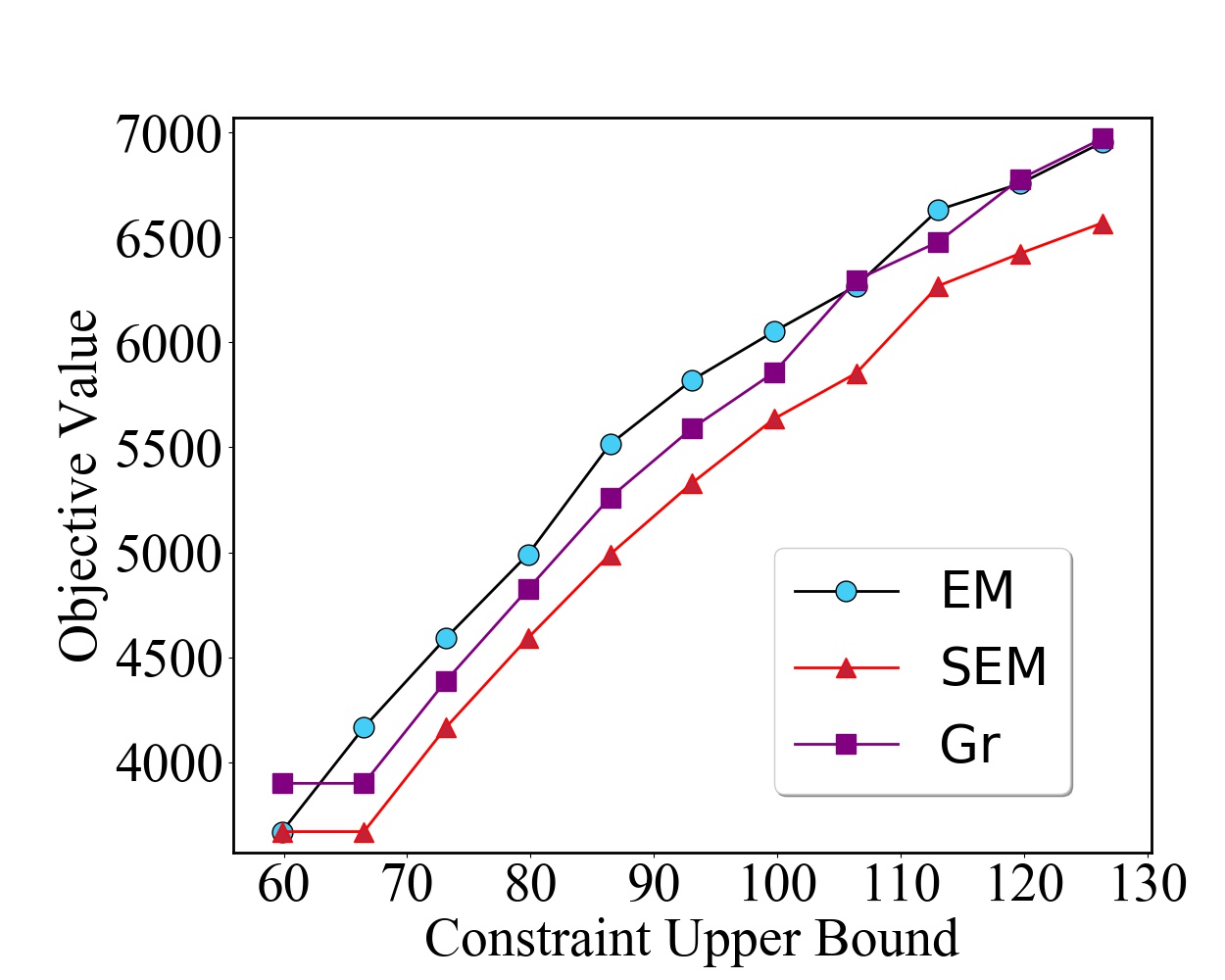}
\end{minipage}
}%
\centering
\caption{Comparison of objective values with public data sets. \label{fig:FL_obj}}
\end{figure*}

\begin{figure*}[hbt!] 
\centering
\subfigure[Upper bound: 65.]{
\begin{minipage}[t]{0.26\linewidth}
\centering
\includegraphics[width=1.8in]{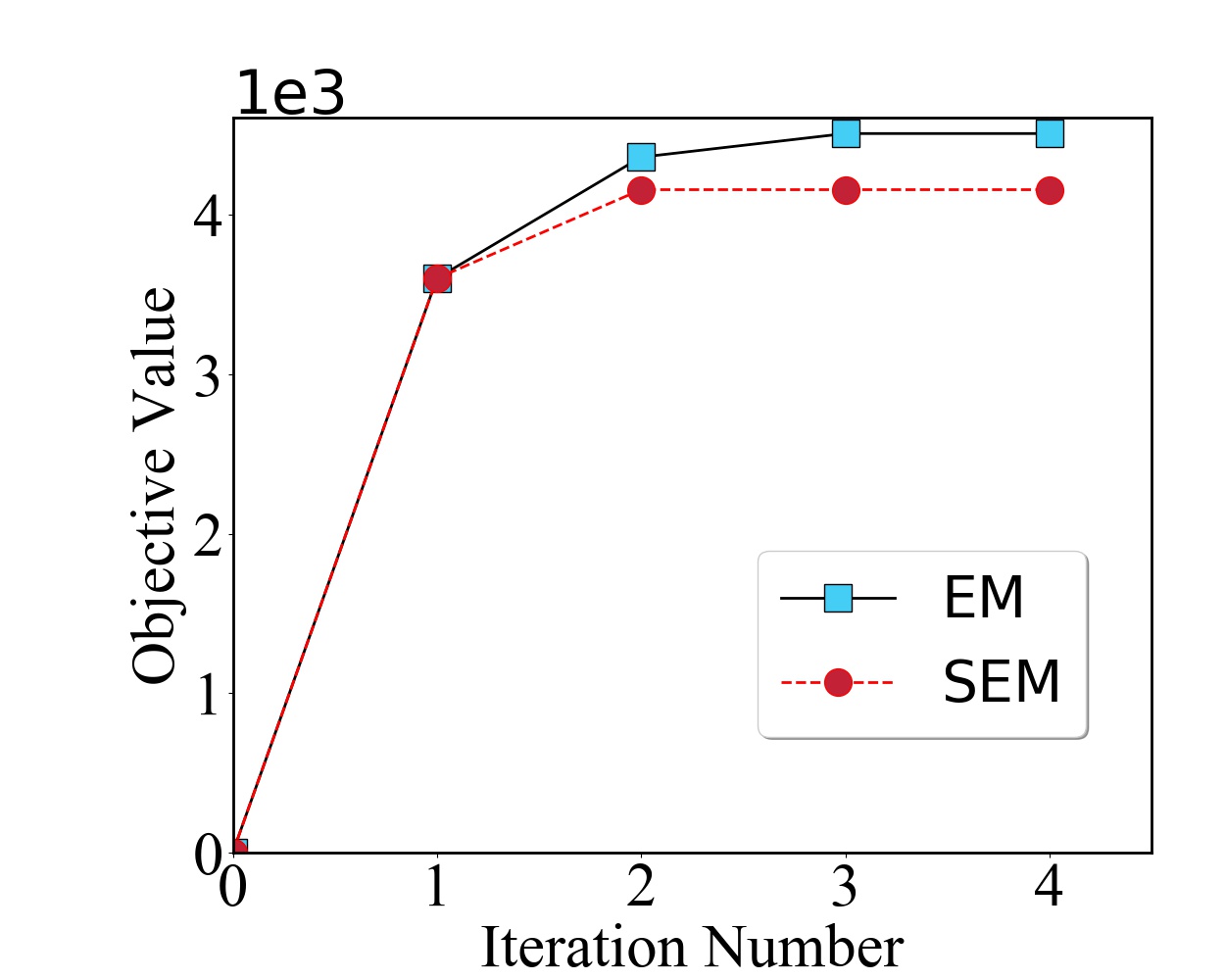}
\end{minipage}%
}%
\subfigure[Upper bound: 70.]{
\begin{minipage}[t]{0.26\linewidth}
\centering
\includegraphics[width=1.8in]{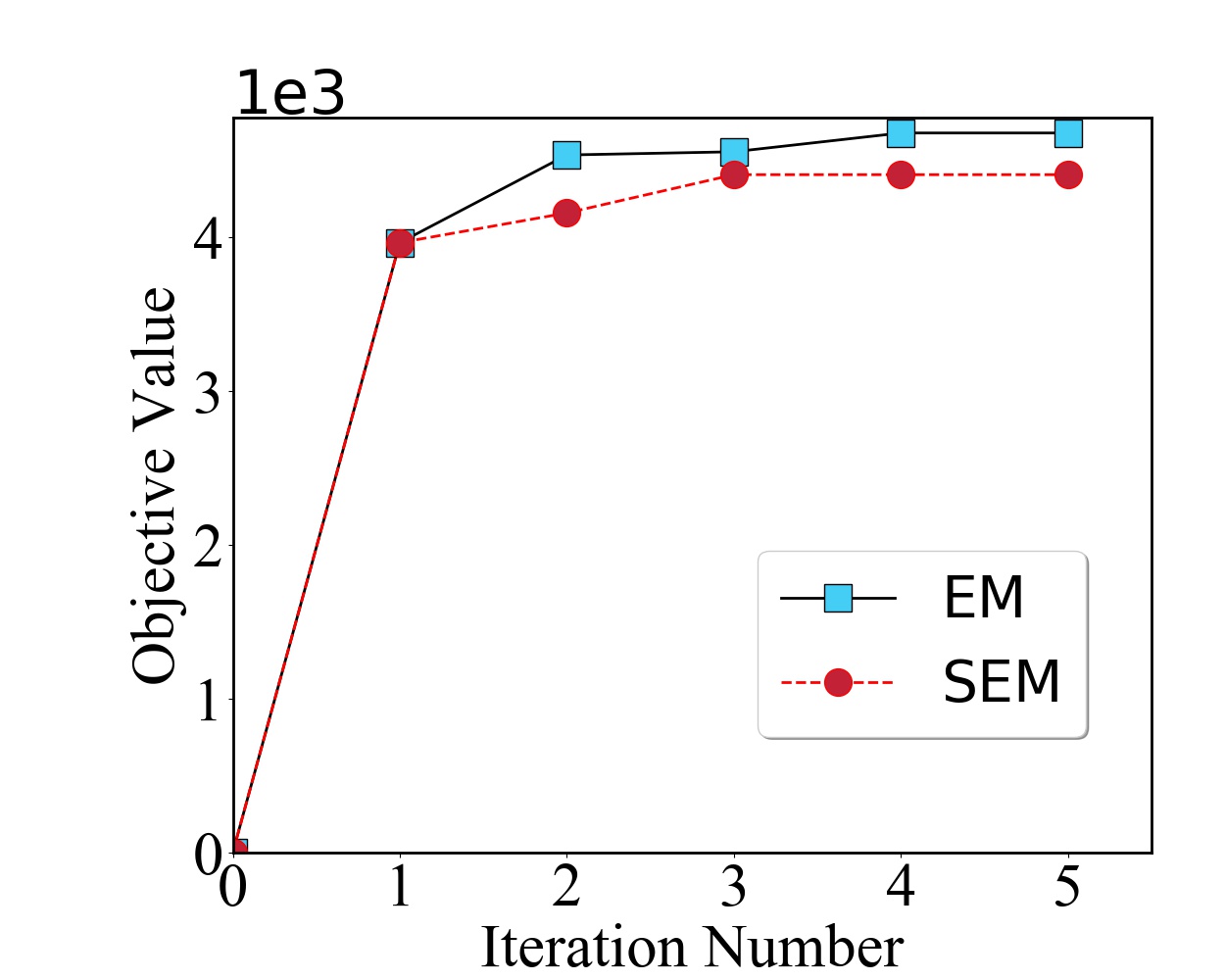}
\end{minipage}%
}%
\subfigure[Upper bound: 85.]{
\begin{minipage}[t]{0.26\linewidth}
\centering
\includegraphics[width=1.8in]{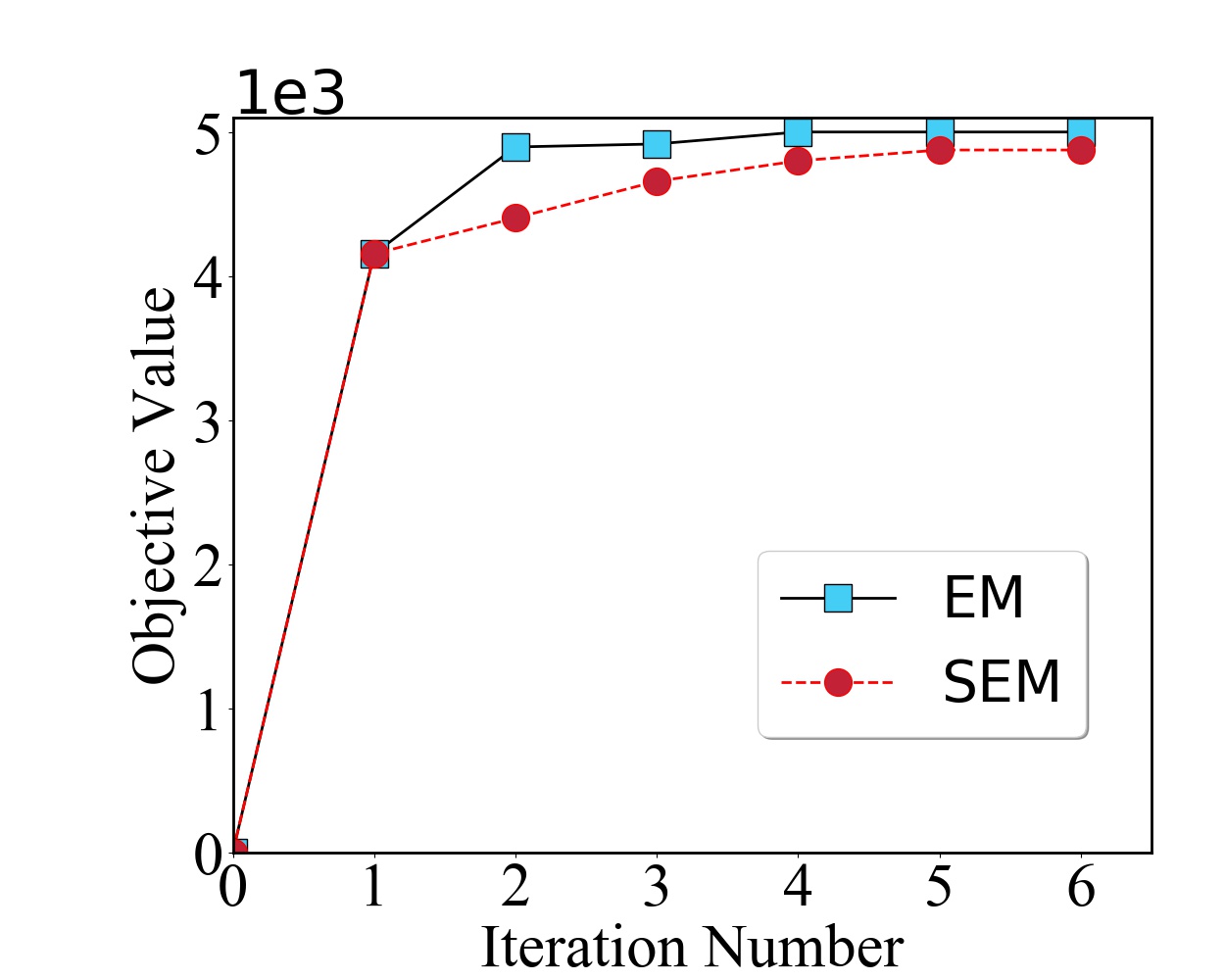}
\end{minipage}
}%
\subfigure[Upper bound: 90.]{
\begin{minipage}[t]{0.26\linewidth}
\centering
\includegraphics[width=1.8in]{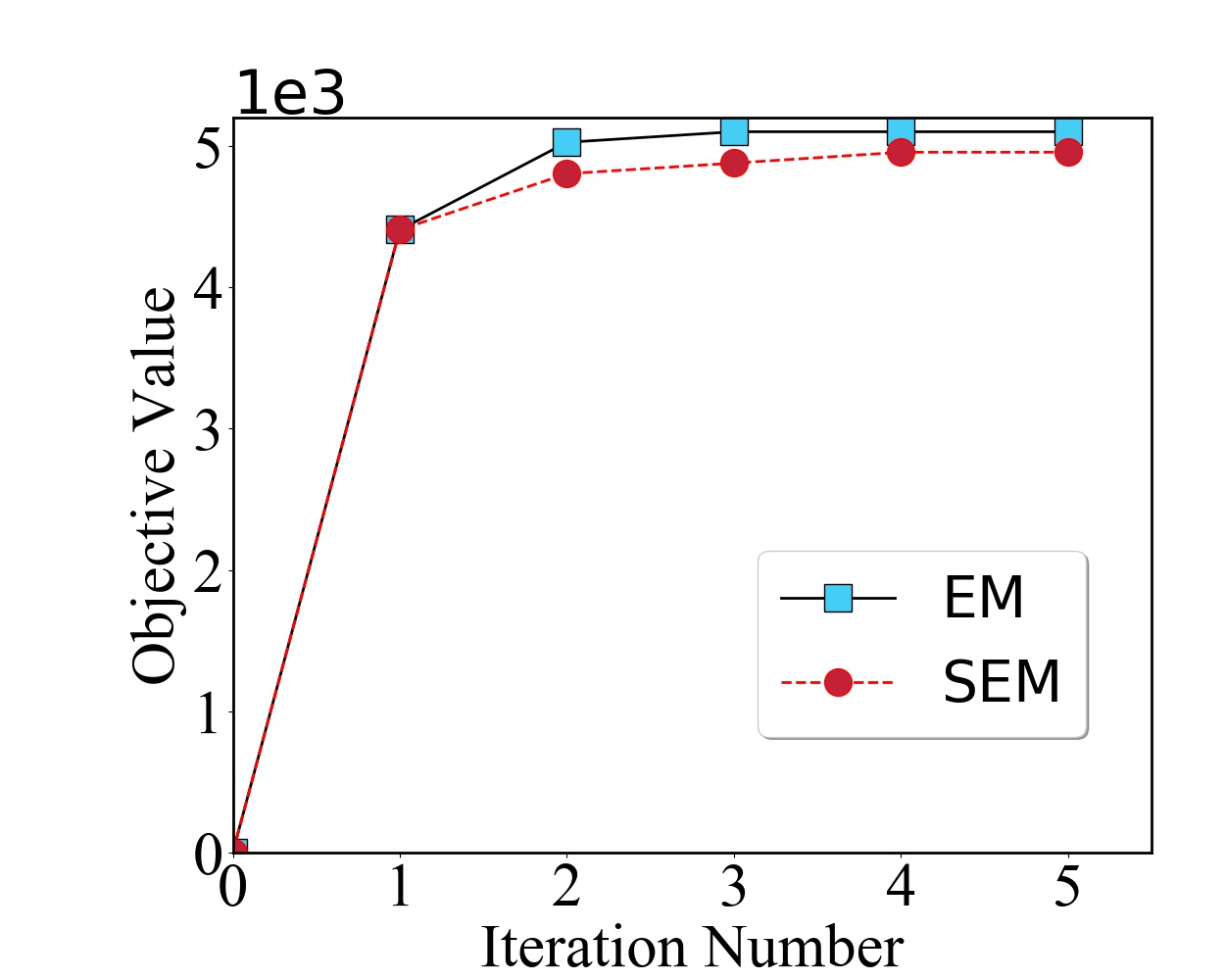}
\end{minipage}
}%
\centering
\caption{Convergence and monotonicity of EM and SEM algorithms with different upper-bound constraints for the Gap-A data set.}
\label{fig:FL_iter}
\end{figure*}
\begin{figure}[hbt!]
\centering
\includegraphics[width=0.5\textwidth]{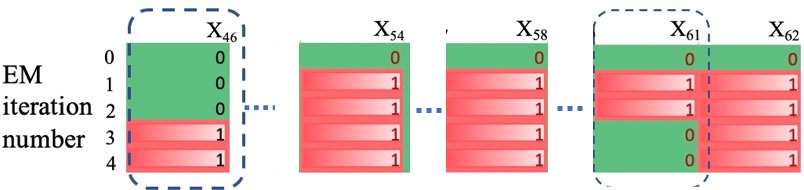}
\caption{${X_i}$ updating process of the EM algorithm. In this example, the $61$-th element was first selected in the $1$-st iteration and later removed in the $3$-rd iteration.}
\label{fig:visual_iter}
\end{figure}

\begin{figure*}[htbp!]
\centering
\subfigure[Data Set 1]{
\begin{minipage}[t]{0.26\linewidth}
\centering
\includegraphics[width=1.7in]{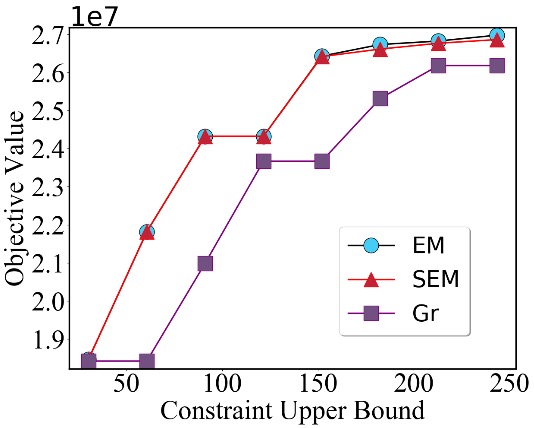}
\end{minipage}%
}%
\subfigure[Data Set 2]{
\begin{minipage}[t]{0.26\linewidth}
\centering
\includegraphics[width=1.7in]{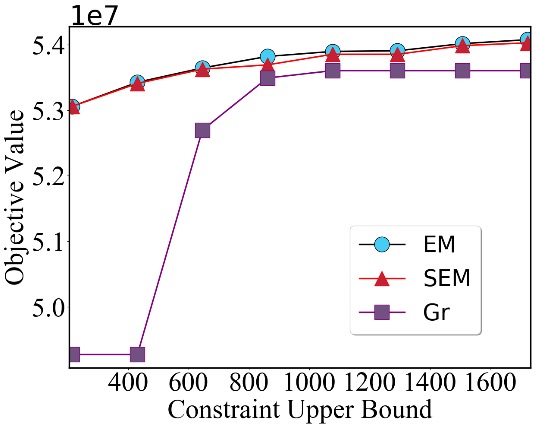}
\end{minipage}%
}%
\subfigure[Data Set 3]{
\begin{minipage}[t]{0.26\linewidth}
\centering
\includegraphics[width=1.7in]{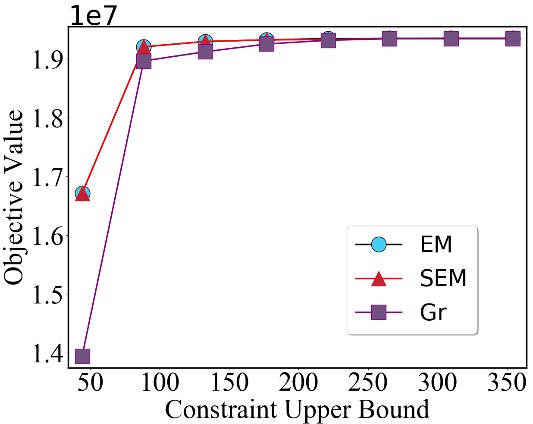}
\end{minipage}
}%
\subfigure[Data Set 4]{
\begin{minipage}[t]{0.26\linewidth}
\centering
\includegraphics[width=1.7in]{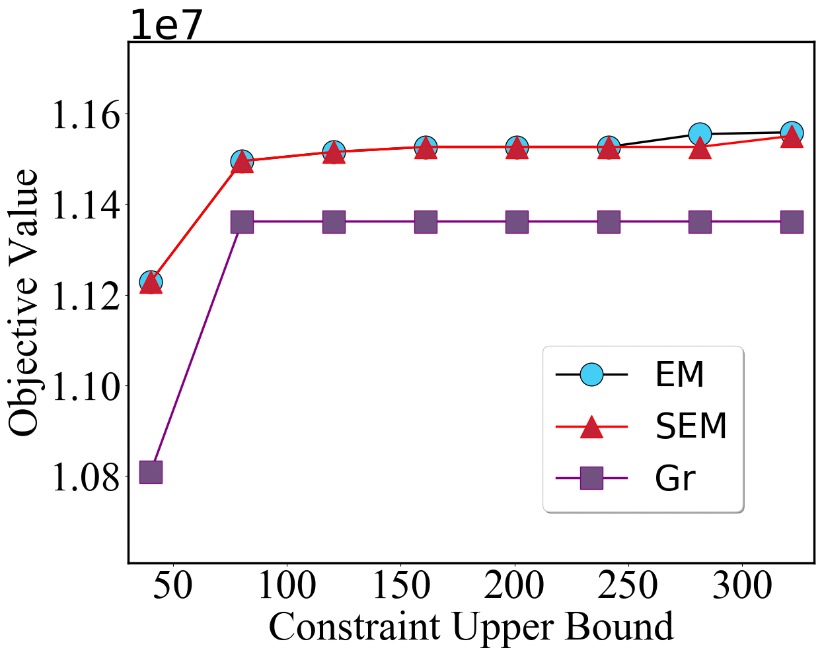}
\end{minipage}
}%
\centering
\caption{Comparison of objective values with different settings for the $4$ different data sets.}
\label{realdata}
\end{figure*}

\begin{figure}[htbp!]
\centering
\includegraphics[width=0.49\textwidth]{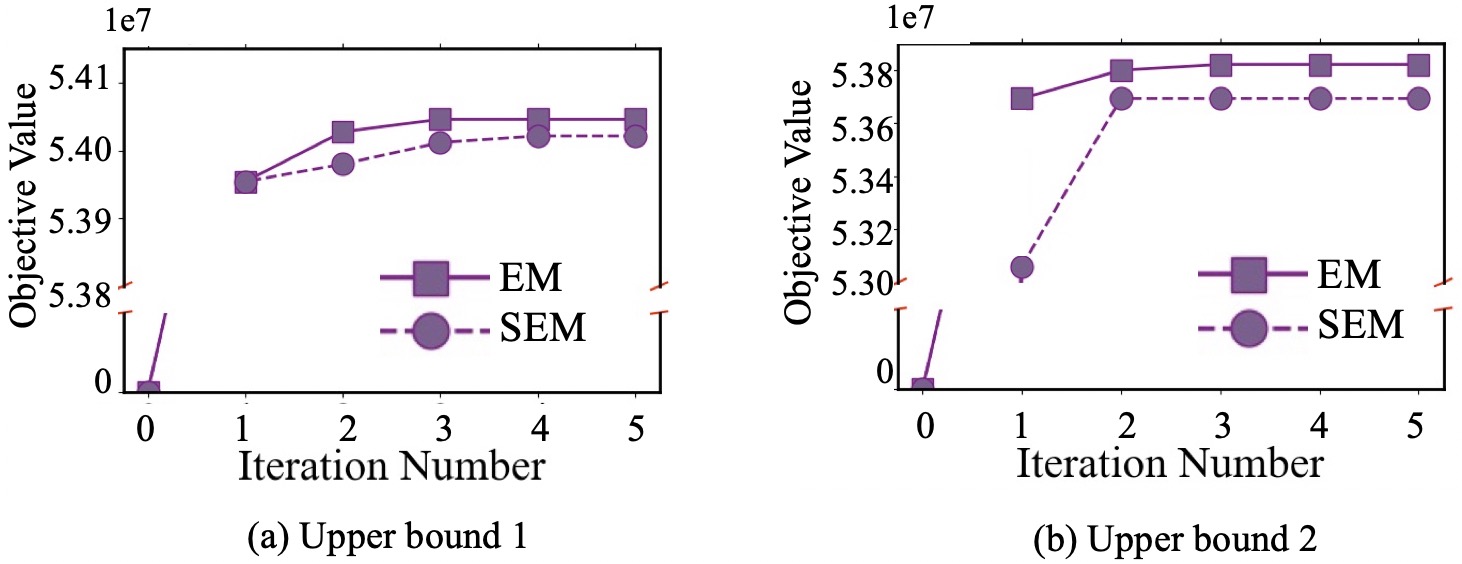}
\caption{Convergence and monotonicity of EM and SEM algorithms with two different upper-bound constraints. }
\label{fig:iter_obj2}
\end{figure}

Since first proposed in \cite{iyer2013submodular},  the SMCP  has been identified for a breadth of applications ranging from training the most accurate classifier subject to process unfairness 
constraints \cite{grgic2018beyond}, automatically designing convolutional neural networks to maximize accuracy within a given forward
time  constraint
\cite{hu2019automatically}
to shifting opinions in a social network through leader selection \cite{yi2019shifting}.

In order to understand the mechanism,  effectiveness, and application potential of the proposed variational framework and EM algorithm for SMCP, 
we start on the public data set and demonstrate the performance advantages over existing methods.
After that, we test the performance in the production environment, first on decision rule selection for fraud transaction detection, and then go further to train a interpretable  classifier
that covers truth positive
in the feature space well, and 
control the false positive within a predefined bound due to production requirement. 

\subsection{Performance on Discrete Location  Data Sets}
To compare the performance of our EM algorithm with that of existing methods, 
we consider four bipartite graphs from the public discrete location data sets \cite{FLdata} including an instance on perfer codes (PCodes), an instance on chess-board (Chess),
an instance on finite projective planes (FPP), and
an instance on large duality gap (Gap-A) 
with
$128$, $144$, $133$, and $100$ nodes for each type of the corresponding bipartite graphs, respectively.
For more detailed information  of these data sets, please refer to \cite{FLdata}. 
Fig. \ref{fig:bipartite_graph} is a running example of this test.
A random value, which is uniformly sampled from $1$ to $100$, is assigned to each circle node, and our goal is to choose a subset of the square nodes to maximize the total sum-value of the covered circle nodes subject to an upper bound constraint of the total number of the square nodes.  

We compare the greedy (Gr) algorithm, which was proposed in the classical work \cite{iyer2013submodular} and has been widely applied for different applications.
Without an E step, it is analogous to the M step in the EM algorithm with a permutation $\epsilon$ such that
\begin{equation}
\widetilde{\epsilon}(i) \in \operatorname{argmax}\left\{g\left(j | S_{i-1}^{\pi}\right) | j \notin S_{i-1}^{\pi}, f\left(S_{i-1}^{\pi} \cup\{j\}\right) \leq b\right\}.
\end{equation}
It was shown that Gr shows best performance in most experiments in \cite{iyer2013submodular}.
We, therefore, compare the EM algorithm with the Gr as well as the SEM algorithms.
The ellipsoidal approximation method in \cite{iyer2013submodular} is not applied here due to high computational complexity.

By considering $11$ upper bounds in each kind of data set,  we thus compare the performances of different algorithms in a total of $44$ experiments.
As shown in Fig.~\ref{fig:FL_obj}, our EM algorithm outperforms all other methods in  all the $44$ experiments except the only case when the upper bound is $60$ in the FPP data set.
Gr algorithms and SEM have overlaps with each other in some settings, yet most of the time Gr outperforms SEM.
Interestingly, in the sub-figure (b), we notice that the Gr algorithm's objective values cannot be increased when the constraint upper bound is increased from $55$ to $60$ as well as from $65$ to $70$. SEM also suffers from the same problem when the constraint upper bound is increased from $50$ to $55$, $60$ to $65$, and $80$ to $85$. Similar problems can also be identified for Gr and SEM in other data sets. However, it is rare to happen to EM. Thus, the experiment demonstrates that our EM algorithm, which enlarges the approximate feasible space in the E step, makes a better use of the feasible space of the SMCP. 

We further test the convergence rate and the monotonicity of the EM algorithm by
fixing the data set to be Gap-A and choosing four different upper bounds, i.e.,
$65$, $70$, $85$, and $90$.
Fig.~\ref{fig:FL_iter}
 shows the objective value versus EM/SEM iteration number.
It demonstrates that
our EM algorithm converges quickly within $3$-$5$ iterations, and the objective value increases monotonically, which is consistent with Proposition 2.
Note that as the initial values are set to be $\emptyset$, the first updates of EM and SEM are
the same and hence the corresponding objective values after first iterations are the same.

To get deeper insight of the EM algorithm, we further demonstrate part of the updating process of it. We take the computing for PCodes data set as an example and set the bound to be $60$. As shown in Fig.~\ref{fig:visual_iter},
 $X_{61}$ was first selected in the first and second iterations and later was removed in the third iteration.
In contrast to the Gr algorithm, which keeps expanding the solution set by adding new element, our EM algorithm select solutions dynamically. This dynamic provides  the capability to obtain a better result.

\setlength\tabcolsep{6pt}
\centering
\begin{table*}[htbp!]
\caption{Summary of transaction data set}
\label{table:close_table}
\begin{tabular}{@{}lcccccccccc@{}}
\toprule
\ \# samples & \# fraud & \# normal & \# features & \# categorical features & \# continuous features \\ \midrule
\ 50,357 & 369 & 49,988 & 50 & 26 & 24 \\
 \bottomrule
\end{tabular}
\label{table:data}
\end{table*}

\setlength\tabcolsep{3pt}
\begin{table}[htbp!]
\caption{Classification performance}
\label{table:classifer}
\centering
\begin{tabular}{@{}lccccc@{}}
\toprule
\textbf{Method} & Fraud coverage & Interruption rate \\ \midrule
Decision Tree & 82.16$\%$ & 1.05$\%$  \\
EM & {83.73$\%$} & {0.96$\%$}   \\ \bottomrule
\end{tabular}
\end{table}

\justify
\subsection{Performance on Fraud Detection Data Sets}
Beyond the public data set experiment, we go further to a practical applications. In our online payment systems, we have a bunch of rules for detecting fraud transactions. Some of them are obtained based on humans experience, and some of them are given by machine learning models like decision tree. Some of these rules could be too aggressive that not only covers the frauds   but also interrupt a lot of normal transactions. Our goal is to select rules that can cover as much fraud  amounts as possible and in the meantime to make sure the interrupted  transaction amount below a predefined value.

We consider four data sets in four different local areas, where each area has their own detection rules due to different attributes in each area.
Each data set consists of transaction index and their labels (fraud or not), and the list of rules that cover each transaction.
There are in total of $1200$, $10052$, $1600$, and $2400$ transactions in each data set, and the number of rules are $85$, $98$, $112$, and $92$ respectively.

As shown in Fig.~\ref{realdata} that our EM algorithm outperforms all other methods consistently for different upper bounds as well  as in different data sets. Furthermore, Fig.~\ref{fig:iter_obj2} shows the objective value of  as a function of the iteration number for data set 2 with different upper bounds. It is demonstrated that our EM algorithm also converges quickly within $3$-$5$ iterations in the industrial environment.

\subsection{Application to Interpretable Classifier}
Following the same context of fraud  detection, we further go beyond the rules selection scenario by modeling the problem as designing an iterpretable classifier based on SMCP.
From the lens of the classier, we are interested in maximizing the true positive subject to an upper bounded false negative. 

More specifically,
given a bunch of features for each transactions, we first apply the efficient F-P algorithm \cite{han2000mining} for mining the frequent fraud transaction patterns/rules.
We limit the maximum rule length to be $4i$ to make it more interpretable.  
Let $[n]$ denote the set of rules obtained.
To detect as many fraud value as possible (which is equivalently to maximize the truth positive),  we 
maximize the following objective function:
$$g(X) =  v\left(
    \cup_{i\in X}\Ccal\Rcal(i)\right),$$
where $X\subseteq [n]$, $\Ccal\Rcal(r_i)$ is the set of frauds covered by rule $i$, and $v(\cdot)$ is the total amount of fraud transaction value covered by $X$.
Moreover, the number of interrupted transactions, i.e., normal transactions but classified mistakenly, can be denoted by 
$$f(X)=  \abr{\cup_{i\in X}\Ccal(i)\setminus \mathcal{CR}(i) },$$
where $\mathcal{C}(i)$ denotes all the transactions covered by rule $i$, either correctly or wrongly.
We then can train a classifier that is consist of the rules selected by maximizing $g(X)$ subjective to an constraint that $f(X)<b$. 
According to the submodularity definition in  (\ref{submodularity2}),  both $f(X)$ and $g(X)$ are  monotonic submodular functions.
Consequently, training the classifer is equivalent to solving an SMCP. We therefore apply our EM algorithm to train this classifier.

We summarize the data set in Table~\ref{table:data} and  split  $75\%$ of the data into a training set and $25\%$ of the data into a testing set. For performance comparison,
we choose a decision tree with a maximum depth of $4$.
We summarize the result in Table~\ref{table:classifer}. It shows that the EM based method achieves performance that covers  more fraud amount and also achieves less 
interruptions.
The advantages could come from the formulation that  builds the classifier, which
exchanges false positive and true negative to identify  as many frauds as possible in the feasible space.

\section{Conclusions}
In this paper, we have proposed a novel variational frame based on the Namhauser divergence for the submodular maximum coverage problem (SMCP). The proposed estimation-and-maximization (EM) method monotonically improves optimization performance in a few iterations.
We have further    
proved a curvature dependent approximate factor for the EM method.
Empirical results on both public data sets and industrial problems in production environment have shown evident performance  improvement over state-of-the-art algorithms.

%

\appendix
\section{Proof of Proposition \ref{KP}}
For the independence of the appendix, we repeat (\ref{sorting2}) with the first $\widehat m+1$ terms as below.
\begin{equation}
\frac{\widehat g^{\pi}_{X_t}(\mu_1)}{ f(\mu_1|M\backslash{\mu_1})}
\geq 
\frac{\widehat g^{\pi}_{X_t}(\mu_2)}{ f(\mu_2|M\backslash{\mu_2})}
\ldots
\geq
\frac{\widehat g^{\pi}_{X_t}(\mu_{\widehat m+1})}{ f(\mu_{\widehat m+1}|M\backslash{\mu_{\widehat m+1}})}.
\end{equation}
Then by induction we have
\begin{equation}
\label{ineqA1}
\frac{\sum_{k=1}^{\widehat m+1}\widehat g^{\pi}_{X_t}(\mu_{k})}{\sum_{k=1}^{\widehat m+1}
 f(\mu_{k}|M\backslash{\mu_{k}})}
\geq 
\frac{\widehat g^{\pi}_{X_t}(\mu_{\widehat m+1})}{ f(\mu_{\widehat m+1}|M\backslash{\mu_{\widehat m +1}})}.
\end{equation}
According to the definition of $\widehat m$ in (\ref{m_max}), it is evident that $$\sum_{k=1}^{\widehat m+1}
f(\mu_{k}|M\backslash{\mu_{k}})\geq b.$$
Substituting the above inequality to (\ref{ineqA1}), it holds that
\begin{equation}
\begin{split}
g^{\pi}_{X_t}(\mu_{\widehat m+1})
&\leq \frac{1}{b}
 f(\mu_{\widehat m+1}|M\backslash{\mu_{\widehat m +1}})
\sum_{k=1}^{\widehat m+1}g^{\pi}_{X_t}(\mu_{k})\\
&
\leq 
\frac{1}{b}
 f(\mu_{\widehat m+1}|\emptyset)
\sum_{k=1}^{\widehat m+1}g^{\pi}_{X_t}(\mu_{k})\\
&
\leq 
\frac{\Delta_f}{b}
\sum_{k=1}^{\widehat m+1}g^{\pi}_{X_t}(\mu_{k}).
\end{split}
\end{equation}
The second inequality is due to (\ref{monotone}), and the third inequality follows from $f(\emptyset)=0$  as well as (\ref{delta_f}).
By subtracting   $\frac{\Delta_f}{b}g^{\pi}_{X_t}(\mu_{\widehat m+1})$ on the left-hand side and then adding $\sum_{k=1}^{\widehat m}g^{\pi}_{X_t}(\mu_{k})$ on both sides, 
we obtain
\begin{equation}
\begin{split}
\sum_{k=1}^{\widehat m}
g^{\pi}_{X_t}(\mu_{k})
+
(1-\frac{\Delta_f}{b})
g^{\pi}_{X_t}(\mu_{\widehat m+1})
\leq 
\frac{\Delta_f}{b}
\sum_{k=1}^{\widehat m+1}g^{\pi}_{X_t}(\mu_{k})
+
\sum_{k=1}^{\widehat m}g^{\pi}_{X_t}(\mu_{k}),
\end{split}
\end{equation}
which equals
\begin{equation}
\begin{split}
\sum_{k=1}^{\widehat m+1}
g^{\pi}_{X_t}(\mu_{k})
-\frac{\Delta_f}{b}
g^{\pi}_{X_t}(\mu_{\widehat m+1})
\leq 
\frac{\Delta_f}{b}
\sum_{k=1}^{\widehat m+1}g^{\pi}_{X_t}(\mu_{k})
+
\sum_{k=1}^{\widehat m}g^{\pi}_{X_t}(\mu_{k}),
\end{split}
\end{equation}
The above inequality still holds after subtracting a positive    $\sum_{k=1}^{\widehat m}
g^{\pi}_{X_t}(\mu_{k})$ on the left-hand side:
\begin{equation}
\begin{split}
(1-\frac{2\Delta_f}{b})
\sum_{k=1}^{\widehat m+1}
g^{\pi}_{X_t}(\mu_{k})
\leq 
\sum_{k=1}^{\widehat m}g^{\pi}_{X_t}(\mu_{k}).
\end{split}
\end{equation}
From (\ref{obj_t}) $\sum_{k=1}^{\widehat m}g^{\pi}_{X_t}(\mu_{k}) = \widehat g(X_{t+1})$, and considering the fact that $\sum_{k=1}^{\widehat m+1}
g^{\pi}_{X_t}(\mu_{k})\geq \textrm{OPT}_{\widehat g}$, we finally prove Proposition \ref{KP}:  
$$\widehat g(X_{t+1})\geq (1-\frac{2\Delta_{f}}{b}) \widehat g({\textrm{OPT}_{\widehat g}}).$$

\end{document}